\documentclass{article}
\usepackage{kr}

\usepackage{times}
\usepackage{soul}
\usepackage{url}
\usepackage[hidelinks]{hyperref}
\usepackage[utf8]{inputenc}
\usepackage[small]{caption}
\usepackage{graphicx}
\usepackage{amsmath}
\usepackage{amsthm}
\usepackage{booktabs}
\usepackage{algorithm}
\usepackage{algorithmic}
\newtheorem{example}{Example}
\newtheorem{theorem}{Theorem}
\newtheorem{proposition}{Proposition}
\newtheorem{lemma}{Lemma}

\usepackage{amssymb}
\usepackage{tikz}
\usepackage{bm}

\renewcommand{\phi}{\varphi}
\renewcommand{\epsilon}{\varepsilon}
\newcommand{\Atoms}{A} 
\renewcommand{\L}{\mathcal{L}} 
\newcommand{\U}{\mathcal{U}} 
\newcommand{\W}{\mathcal{W}} 
\newcommand{\dual}[1]{\overline{#1}} 

\newcommand{\hamming}{{\mathrm{H}}} 

\newcommand{\Neutrality}{\mathtt{N}} 
\newcommand{\Flipping}{\mathtt{F}} 
\newcommand{\Addition}{\mathtt{A}} 
\newcommand{\BoB}{\mathtt{BOB}} 
\newcommand{\BoW}{\mathtt{BOW}} 
\newcommand{\BoWS}{\mathtt{BOWS}} 

\title{Surprise Minimization Revision Operators}

\author{%
    Adrian Haret
    \affiliations
    Institute for Logic, Language and Computation\\ The University of Amsterdam
    \emails
    a.haret@uva.nl    
}

\begin{document}

\maketitle

\begin{abstract}
	Prominent approaches to belief revision prescribe the adoption
	of a new belief that is as close as possible to the prior belief, 
	in a process that, even in the standard case, can be described as attempting to minimize surprise.
	Here we extend the existing model by proposing a measure of surprise, dubbed \emph{relative surprise},
	in which surprise is computed with respect not just to the prior belief, 
	but also to the broader context provided by the new information, 
	using a measure derived from familiar distance notions between 
	truth-value assignments. We characterize the surprise minimization revision operator thus defined 
	using a set of intuitive rationality postulates in the AGM mould, 
	along the way obtaining representation results for other existing revision operators
	in the literature, such as the Dalal operator and a recently introduced distance-based min-max operator. 
\end{abstract}

\section{Introduction}
Belief change models rational adjustments made to an agent's epistemic state 
upon acquiring new information \cite{Peppas2008,Hansson2017,FermeH2018}.
When the new information is assumed to be reliable,
the logic of changing one's prior beliefs to accommodate such new-found knowledge 
falls under the heading of \emph{revision}. 
Belief revision is typically thought of by appeal to a set of intuitive normative 
principles, usually along the lines of the AGM framework \cite{AlchourronGM1985}, 
alongside more concrete revision representations and mechanisms 
\cite{Grove1988,Dalal1988,GardenforsM88,KatsunoM1992,Rott1992}.

A perspective underlying many of these representations, which we share here, 
is that belief revision is akin to a choice procedure 
guided by a plausibility relation over possible states of affairs:
revising a belief, in this sense, amounts to choosing the most plausible states of affairs consistent
with the new information.
Plausibility over states of affairs, in turn, is judged according to some notion of 
dissimilarity, or distance between states of affairs: 
I judge a situation to be less likely the further away from my own belief it is.
Among the various distance notions that can be used to make this intuition precise, 
the approach using Hamming distance to rank truth-value assignments is among the most prominent, 
used for the well-known Dalal revision operator \cite{Dalal1988}, 
and the more recently introduced Hamming distance min-max operator \cite{HaretW2019}. 



Both the Dalal and the Hamming distance min-max operator are designed to respond to new information
by minimizing departures from the prior belief, 
in what can be described, just as well, as an attempt to prevent major surprise:
if I have a prior belief that all major carbon emitting countries will have halved their emissions 
by the end of 2049, 
and it turns out that neither of them has, 
then I am likely to be surprised---certainly more suprised than seeing my belief confirmed.
Consequently, if I acquire information to the effect that these are the only two possible outcomes
(i.e., either all countries cut emissions, or none of them does), 
then, on the assumption that this information stems from some noisy observation of the true state,
I will use my prior belief and gravitate towards the outcome that occasions less surprise.

In this revision procedure, consistent with both the Dalal and the min-max operators, 
the measure of surprise is taken to depend only on the absolute difference between my prior belief
and the states of affairs learned to be viable.
However, we can readily imagine that the amount of anticipated surprise depends in equal measure 
on other factors, e.g., the context provided by the newly acquired information: 
if in 2049 it turns out that none of the countries has reduced emissions, 
then I am likely to be less surprised 
if I had been told in advance that at most one of them would
than if I had been told that, possibly, any number of them could achieve the target.
In other words, it is desirable to have a broader notion of surprise complementing the absolute one, 
to account for situations in which change in the epistemic state 
depends not just on the prior belief but also on the range of options provided by the new information.
However, despite the fact that surprise minimization is a natural idea
that has been gaining traction in Cognitive Science \cite{Friston2010,Hohwy2016},
there are not many belief revision policies that explicitly take it into account.

In this paper we put forward a notion of relative surprise that is richer in precisely this sense,
and leverage it to define a new type of revision operator, 
called the \emph{Hamming surprise min-max operator},
and which is calibrated to take into account contextual effects as described above.
Though it deviates from some of the postulates in the AGM framework 
(notably, \emph{Vacuity}, \emph{Superexpansion} and \emph{Subexpansion} \cite{FermeH2018}), 
we show that the Hamming surprise min-max operator shares other desirable, 
though less obvious, features with the Dalal and the Hamming distance min-max operator.
Significantly, we use these features to fully characterize the newly introduced surprise operator, 
in the process obtaining full chacterizations for the Dalal and Hamming distance min-max operators.

\paragraph{Contributions.}
On a conceptual level, we argue that the notion of distance standardly used to define revision operators
can be seen as quantifying a measure of surprise, 
with different distance-based operators providing different ways to minimize it.
We then enrich this landscape by introducing a notion of \emph{relative} surprise, 
which is then put to use in defining the Hamming surprise min-max operator. 
We compare this operator against the standard KM postulates for revision 
\cite{KatsunoM1992} and present new postulates that complement the KM ones, 
for a full characterization. 
The versatility of the ideas underlying these postulates is showcased
by adapting them to the Dalal and Hamming distance min-max operators: 
in the case of the min-max operator our postulates complement the subset of KM postulates
the operator is known to satisfy;
in the case of the Dalal operator our postulates strengthen the KM postulates.
In both cases, we obtain full characterizations. 

\paragraph{Related work.}
Among belief revision operators that are insensitive to syntax, 
the Dalal operator has received a significant amount of attention,
either from attempts to express it by encoding the Hamming distance between truth-value assignments
at the syntactic level \cite{Val1993,PozosLP2013};
as an instance of the more general class of parameterized difference operators
\cite{PeppasW2018,AravanisPW21};
or in relation to Parikh's relevance-sensitivity axiom \cite{PeppasWCF2015}.
However, to the best of our knowledge, the characterization we offer here is the first of its kind.

Strengthening the AGM framework to induce additional desired behavior from revision operators 
has been considered in relation to issues of iterated revision \cite{DarwicheP1997}, 
or relevance sensitivity \cite{Parikh99,PeppasW2016}. 
In terms of choice rules, the closest analogue to the surprise minimization operator 
is the decision rule that minimizes maximum regret in decisions with ignorance 
\cite{Milnor1954,LaveM1993,Peterson2017}, with Hamming distances playing the role of  
utilities in our present setting. However, the logical setting and the fact that the distances depend on the states 
themselves means that decision theoretic results do not translate easily to our current framework.

\paragraph{Outline.}
Section \ref{sec:preliminaries} introduces the main notions related to propositional logic and 
belief revision that will be used in the rest of the paper, 
and argues for the surprise-based interpretations of distances,
Section \ref{sec:distance-surprise-operators} defines the relative Hamming surprise measure 
and the Hamming surprise min-max operator. 
Sections \ref{sec:dalal} and \ref{sec:distance-hamming-max} consist of a slight detour in which 
the Dalal and Hamming distance min-max operators are characterized, setting up the stage 
for the characterization of the surprise operator in Section \ref{sec:surprise-hamming-max}.
Section \ref{sec:conclusion} offers conclusions.

\section{Preliminaries}\label{sec:preliminaries}
\paragraph{Propositional Logic.}
We assume a finite set $\Atoms$ of \emph{propositional atoms},
large enough that we can always reach into it and find additional, unused atoms, 
if any are needed.
The \emph{set $\L$ of propositional formulas} is generated from the atoms in $\Atoms$
using the usual propositional connectives ($\land$, $\lor$, $\lnot$, $\rightarrow$ and $\leftrightarrow$),
as well as the constants $\bot$ and $\top$. 

An \emph{interpretation $w$} 
is a function mapping every atom in $\Atoms$ to either \emph{true} or \emph{false}.
Since an interpretation $w$ is completely determined 
by the set of atoms in $\Atoms$ it makes true,
we will identify $w$ with this set of atoms
and, if there is no danger of ambiguity, display $w$ 
as a word where the letters are the atoms assigned to true.
The \emph{universe $\U$} is the set of all interpretations 
for formulas in $\L$.
If $w_1$ and $w_2$ are interpretations, the \emph{symmetric difference $w_1\triangle w_2$ of $w_1$ and $w_2$}
is defined as $w_1\triangle w_2 = (w_1\setminus w_2)\cup (w_2\setminus w_1)$, 
i.e., as the set of atoms on which $w_1$ and $w_2$ differ.
The \emph{Hamming distance $d_\hamming\colon \U\times\U\rightarrow\mathbb{N}$} is defined, 
for any interpretations $w_1$ and $w_2$, as 
$
	d_\hamming(w_1,w_2) =|w_1\triangle w_2|.
$
Intuitively, the Hamming distance $d_\hamming(w_1,w_2)$ between $w_1$ and $w_2$
counts the number of atoms that $w_1$ and $w_2$ differ on,
and is used to quantify the disagreement between two interpretations.

The models of a propositional formula $\phi$ are the interpretations that satisfy it,
and we write $[\phi]$ for the set of models of $\phi$.
If $\phi_1$ and $\phi_2$ are propositional formulas,
we say that \emph{$\phi_1$ entails $\phi_2$}, 
written $\phi_1 \models \phi_2$, if $[\phi_1]\subseteq[\phi_2]$,
and that they are \emph{equivalent}, written $\phi_1 \equiv \phi_2$, if $[\phi_1]= [\phi_2]$.
A propositional formula $\phi$ is \emph{consistent} if $[\phi]\neq\emptyset$.
The models of $\bot$ and $\top$ are $[\bot]=\emptyset$ and $[\top]=\U$.
We will occasionally find it useful to explicitly represent the models of a formula, 
in which case we write $\phi_{v_1,\dots,v_n}$ for a propositional formula 
such that $[\phi_{v_1,\dots,v_n}]=\{v_1,\dots,v_n\}$.
A propositional formula $\phi$ is \emph{complete} if it has exactly one model,
and we will typically denote a complete formula as $\phi_v$ to 
draw attention to its unique model $v$.
The \emph{null formula $\epsilon$} and the \emph{full formula $\alpha$} 
are defined as $\epsilon = \bigwedge_{p\in\Atoms}\lnot p$
and $\alpha = \bigwedge_{p\in\Atoms} p$,
i.e., as the conjunction of the negated and non-negated atoms in $\Atoms$, respectively.
Note that $[\epsilon]=\{\emptyset\}$ and $[\alpha]=\Atoms$.

\paragraph{Distance-based belief revision.}
A \emph{revision operator $\circ$} is a function
$\circ\colon\L\times\L\rightarrow\L$,
taking as input two propositional formulas,
denoted $\phi$ and $\mu$,
and standing for the agent's prior and newly acquired 
information, respectively,
and returning a propositional formula,
denoted $\phi\circ\mu$.
Two revision operators $\circ_1$ and $\circ_2$ are \emph{equivalent}, 
written $\circ_1\equiv\circ_2$, 
if $\phi\circ_1\mu\equiv\phi\circ_2\mu$, 
for any formulas $\phi$ and $\mu$.

The primary device for generating concrete revision operators we make recourse to here is  
the Hamming distance. Thus, the \emph{Hamming distance min-min operator $\circ^{d_\hamming,\,\min}$}, 
or, as it is more commonly known, \emph{the Dalal operator} \cite{Dalal1988},
is defined, for any propositional formulas $\phi$ and $\mu$, 
as a formula $\phi\circ^{d_\hamming,\,\min}\mu$ such that:
$$
	[\phi\circ^{d_\hamming,\,\min}\mu] = \mathrm{argmin}_{w\in[\mu]}\min_{v\in[\phi]} d_\hamming(v,w).
$$
Intuitively, the shortest distance from $w$ to any model of $\phi$,
i.e., $\min_{v\in[\phi]}d_\hamming(v,w)$, can be interpreted as a measure of distance between 
$w$ and $\phi$, and we will refer to it as the \emph{Hamming min-distance between $\phi$ and $\mu$}.
The result $\phi\circ^{d_\hamming,\,\min}\mu$ of revision, then, selects those models of $\mu$
that are closest to $\phi$ according to this measure.

Recently, an alternative revision operator has been analyzed \cite{HaretW2019}: 
what we will call here the \emph{Hamming distance min-max operator $\circ^{d_\hamming,\,\max}$},
defined, for any $\phi$ and $\mu$, 
as a formula $\phi\circ^{d_\hamming,\,\max}\mu$ such that:
$$
	[\phi\circ^{d_\hamming,\,\max}\mu] = \mathrm{argmin}_{w\in[\mu]}\max_{v\in[\phi]} d_\hamming(v,w),
$$
i.e., a formula whose models are exactly those models of $[\mu]$ that 
minimize the Hamming distance to 
$\max_{v\in[\phi]}d_\hamming(v,w)$,
the \emph{Hamming max-distance between $\phi$ and $\mu$}.

\paragraph{Distance as surprise.}
Consistent with the idea that revision models the agent learning about the world around it,
we can see the new information $\mu$ as a noisy observation of some underlying ground truth state $w^\ast$:
by acquiring $\mu$, the agent learns of a set of outcomes (the models of $\mu$),
all of which stand a chance of being the true state $w^\ast$. In that sense, the distance $d(v,w)$
between any $v\in[\phi]$ and $w\in[\mu]$ stands for a quantity that can be aptly described as \emph{surprise}:
it is the difference between what the agent expects is the case ($v$)
and what might turn out to actually be the case ($w$). Naturally, the agent will want to minimize 
the divergence between its predictions and reality, with existing revision operators providing 
different means to do so.

\begin{table}\centering
	\begin{tabular}{rcccc}
		\toprule
		$d_{\hamming}$             & 
		$\emptyset$				   & 
		$abcd$ 					   & 
		$\min$					   &
		$\max$ 					   \\\midrule 
		
	    $\emptyset$				   & 
		$0$     	   			   & 
		$4$	    	   			   & 
		$\bm{0}$               		   &
		$4$       		   \\
		
	    $abcd$ 					   & 
		$4$  		     		   & 
		$0$	    				   & 
		$\bm{0}$       		   &
		$4$ 					   \\

		$abe$ 					   & 
		$3$  		     		   & 
		$3$	    				   & 
		$3$       		   		   &
		$\bm{3}$     			   \\
		\bottomrule
	\end{tabular}	
	\caption{
		Hamming distances $d_\hamming(v, w)$ 
		for $v\in[\phi]$, $w\in[\mu]$,
		with $[\phi]=\{\emptyset,abcd\}$ and $[\mu] = \{\emptyset,abcd, abe\}$.
		The lower $d_\hamming(v, w)$ is, the more plausible $w$ is considered to be, 
		from the standpoint of $v$. 
		The minimal and maximal values per model of $\mu$ are tallied on the right,
		with the values preferred by operators $\circ^{d_\hamming,\,\min}$ and $\circ^{d_{\hamming},\,\max}$,
		i.e., the minimal among the minimal and maximal values, respectively, in bold font.
	}
	\label{tab:dalal-hmax}
\end{table}

\begin{example}\label{ex:dalal-hmax}
	Consider a set $\Atoms=\{a,b,c,d,e\}$ of atoms, 
	standing for countries that might meet their emission targets before 2049,
	and formulas 
	$\phi = (\lnot a\land \lnot b\land \lnot c\land\lnot d\land \lnot e)\lor (a \land b \land c\land d\land \lnot e)$
	and $\mu = \phi\lor (a\land b\land \lnot c\land \lnot d\land e)$,
	with $[\phi] = \{\emptyset, abcd\}$ and $[\mu] = \{\emptyset,abcd,abe\}$.
	Using the Hamming distances depicted in Table \ref{tab:dalal-hmax},
	we obtain that 
	$[\phi \circ^{d_\hamming,\,\min}\mu]=\{\emptyset, abcd\}$
	and $[\phi \circ^{d_\hamming,\,\max}\mu]=\{abe\}$.

	Intuitively, we read this as saying that if an agent believes the 
	true state to be either of the worlds in $[\phi]=\{\emptyset, abcd\}$,
	but finds out it is one among 
	$[\mu] = \{\emptyset, abcd, abe\}$, 
	then $\circ^{d_\hamming,\,\min}$ selects the new belief to be $\{\emptyset,abcd\}$, 
	as this supplies the least amount of surprise in an optimistic, best best-case scenario:
	if the true state turns out to be either of $\emptyset$ or $abcd$, 
	then the agent, believing this, will be able to say ``I told you so!''; 
	the $abe$ case, which is surprising in both cases, is ignored.
	In a complementary approach, the $\circ^{d_\hamming,\,\max}$ 
	operator shifts the agent's belief to $\{abe\}$, 
	as this provides, more cautiously, the best worst-case scenario: 
	from the standpoint of both $\emptyset$ or $abcd$, 
	$abe$ seems the least risky of the other options.
\end{example}

\noindent 
Example \ref{ex:dalal-hmax} serves as a springboard for some important observations.
Firstly, it illustrates that $\circ^{d_\hamming,\,\min}$ and $\circ^{d_{\hamming},\,\max}$
are 
distinct operators.
Secondly, it is apparent from Example \ref{ex:dalal-hmax} that, given prior beliefs $\phi$, 
interpretations can be ranked according to their Hamming min- or max-distance to $\phi$.
It is straightforward to see that
($i$) in both cases the resulting rankings depend only on the models of $\phi$, are total and admit ties;
($ii$) the min-distance places models of $\phi$ at the bottom of this ranking,
i.e., as the most plausible interpretations according to $\phi$, 
in a pattern that goes under the name of a \emph{faithful ranking} \cite{KatsunoM1992};
and, perhaps, less conspicuously, that
($iii$) the max-distance places models of the so-called
\emph{dual of $\phi$} (i.e., the formula obtained from $\phi$ 
by replacing all its atoms with their negations), 
at the very top, i.e., as the least plausible interpretations according to $\phi$ \cite{HaretW2019}.
The different flavors of rankings, faithful or otherwise, generated in this distance-based approach 
usually play a prominent role in representation results for revision, as they open up a level of abstraction
between that of concrete numbers and general principles. In this work, however, we will bypass talk of rankings 
and work directly at the interface between distance-based measures and normative principles.

Finally, an observation that will prove useful is that 
we can (and will) think of the individual models $v$ of $\phi$ as generating 
their own plausibility rankings over interpretations: 
these rankings correspond to the columns in Table \ref{tab:dalal-hmax}
and are the rankings that would be generated if the prior belief were the complete formula $[\phi_v]=\{v\}$, 
i.e., what the landscape of plausibility looks like if the agent 
puts the entire weight of its belief on $\phi_v$.
Revision can then be seen as 
employing a function ($\min$ or $\max$) to aggregate the individual rankings, 
and then choosing something out of the aggregated result: 
the Dalal operator $\circ^{d_\hamming,\,\min}$ chooses, optimistically, the models that are the best of the best, 
while $\circ^{d_{\hamming},\,\max}$ chooses, pessimistically, 
the best of the worst models across the individual rankings.
In keeping with this way of looking at things, we will often speak, loosely, 
of formulas and interpretations `judging' and `choosing' among possible outcomes.

What recommends the choice behavior of operators (such as Dalal's operator) as reasonable
is adherence to a set of intuitive normative principles, or \emph{rationality postulates}. 
The most common set of such principles consists of the AGM postulates for revision \cite{AlchourronGM1985}, 
which we present here in the Katsuno-Mendelzon formulation \cite{KatsunoM1992}. 
The postulates apply for any propositional formulas $\phi$, $\mu$, $\mu_1$ and $\mu_2$: 

\begin{description}
	\item[($\mathrm{R}_{1}$)] $\phi\circ\mu\models\mu$.
  	\item[($\mathrm{R}_{2}$)] If $\phi\land\mu$ is consistent, 
    	then $\phi\circ\mu\equiv\phi\land\mu$.
	\item[($\mathrm{R}_{3}$)] If $\mu$ is consistent, then $\phi\circ\mu$ is consistent.
	\item[($\mathrm{R}_{4}$)] If $\phi_1\equiv\phi_2$ and $\mu_1\equiv\mu_2$, 
    	then $\phi_1\circ\mu_1\equiv\phi_2\circ\mu_2$.
  	\item[($\mathrm{R}_{5}$)] $(\phi\circ\mu_1)\land\mu_2\models\phi\circ(\mu_1\land\mu_2)$.
	\item[($\mathrm{R}_{6}$)] If $(\phi\circ\mu_1)\land\mu_2$ is consistent, 
    	then $\phi\circ(\mu_1\land\mu_2)\models(\phi\circ\mu_1)\land\mu_2$.
\end{description}

\noindent
The primary assumption of revision (postulate $\mathrm{R}_{1}$) 
is that new information originates with a trustworthy source; 
thus, revising $\phi$ by $\mu$ involves a commitment to accept 
the newly acquired information.
Postulate $\mathrm{R}_{2}$, known as the \emph{Vacuity} postulate, 
says that if the newly acquired information $\mu$ 
does not contradict the prior information $\phi$, the result is just the conjunction of $\mu$ and $\phi$.
Postulate $\mathrm{R}_{3}$ says that if the newly acquired information $\mu$ is consistent, 
then the revision result should also be consistent.
Postulate $\mathrm{R}_{4}$ says that the result depends only on the semantic content of the information involved.
Postulates $\mathrm{R}_{5}$ and $\mathrm{R}_{6}$, 
known as \emph{Subexpansion} and \emph{Superexpansion}, respectively, 
enforce a certain kind of coherence when the new information is presented sequentially, 
which is for present purposes best understood as akin to a 
form of \emph{independence of irrelevant alternatives} 
familiar from rational choice \cite{Sen2017}: 
the choice over two alternatives (here, interpretations $w_1$ and $w_2$ in $[\mu]$) 
should \emph{not} depend on the presence of other alternatives in the menu (here represented by new information $\mu$).


The Dalal operator $\circ^{d_\hamming,\,\min}$ satisfies postulates $\mathrm{R}_{1}$-$\mathrm{R}_{6}$
\cite{KatsunoM1992}, though these postulates do not uniquely characterize it.
The Hamming distance max-operator $\circ^{d_{\hamming},\,\max}$ 
satisfies postulates $\mathrm{R}_{1}$ and $\mathrm{R}_{3}$-$\mathrm{R}_{6}$
but not $\mathrm{R}_{2}$, though it does satisfy the following two postulates \cite{HaretW2019},
where $\dual{\phi}$ stands for the \emph{dual of $\phi$}, as defined above:
\begin{description}
	\item[($\mathrm{R}_{7}$)] If $\phi\circ\mu\models\dual{\phi}$, then $\phi\circ\mu\equiv\mu$.
	\item[($\mathrm{R}_{8}$)] If $\mu\not\models\dual{\phi}$, then $(\phi\circ\mu)\land\dual{\phi}$ is inconsistent.
\end{description}
\noindent
In certain circumstances, $\dual{\phi}$ can be thought of as the point of view opposite to that of $\phi$,
such that, taken together, postulates $\mathrm{R}_{7}$ and $\mathrm{R}_{8}$
inform the agent to believe states of affairs
compatible with $\dual{\phi}$ only if it has no other choice in the matter:
the models of $\dual{\phi}$ should be part of a viewpoint one is willing to accept
only as a last resort.

\section{Relative Hamming Surprise Minimization}\label{sec:distance-surprise-operators}
In this section we introduce our novel surprise-based operator.
We start by defining, for any interpretations $v$ and $w$, 
the (relative) \emph{Hamming surprise 
$s_\hamming^{\mu}(v,w)$ of $v$ with respect to $w$ relative to $\mu$}, as:
$$
	s_\hamming^{\mu}(v,w) = d_\hamming(v,w) - d_\hamming(v,\mu),
$$
i.e., the distance between $v$ and $w$ normalized by the distance 
between $v$ and $\mu$. The new information $\mu$, here, 
serves as the reference point, or context, relative to which surprise is calculated.
The \emph{Hamming surprise min-max operator $\circ^{s,\,\max}$} 
is defined as a formula $\phi\circ^{s_{\hamming},\,\max}\mu$ such that:
$$
	[\phi\circ^{s_\hamming,\,\max}\mu] = \mathrm{argmin}_{w\in[\mu]}\max_{v\in[\phi]}s_\hamming^{\mu}(v,w),
$$
i.e., as a formula whose models are exactly those models of $\mu$ 
that minimize maximum Hamming surprise with respect 
to $\phi$, and relative to $\mu$.
We refer to $\max_{v\in[\phi]}s_\hamming^{\mu}(v,w)$, 
as the \emph{max-surprise of $\phi$ with $w$ relative $\mu$}.

\begin{table}\centering
	\begin{tabular}{rccc}
		\toprule
		$s^\mu_{\hamming}$             & 
		$\emptyset$				   & 
		$abcd$ 					   & 
		$\max$ 					   \\\midrule 
		
	    $\emptyset$				   & 
		$0-0$     	   			   & 
		$4-0$	    	   		   & 
		$4$       		   \\
		
	    $abcd$ 					   & 
		$4-0$  		     		   & 
		$0-0$	    				   & 
		$4$ 					   \\

		$abe$ 					   & 
		$3-0$  		     		   & 
		$3-0$   				   & 
		$\mathbf{3}$      		   		   \\
		\bottomrule
	\end{tabular}	
	\caption{
		Relative Hamming surprise $s^\mu_\hamming(v, w)$ 
		for $v\in[\phi]$, $w\in[\mu]$, 
		for $[\phi]=\{\emptyset,abcd\}$, 
		$[\mu] = \{\emptyset,abcd, abe\}$, 
		and relative to $\mu$:
		$d_\hamming(v,w)$ is normalized by the distance $d_\hamming(v, \mu)$ from $v$ to $\mu$.
		The lower surprise is, the more plausible $w$ is considered to be, 
		from the standpoint of $v$. 
		The model minimizing overall surprise is emphasized in bold font.
	}
	\label{tab:hamming-surprise-1}
\end{table}

\begin{table}\centering
	\begin{tabular}{rccc}
		\toprule
		$s^\nu_{\hamming}$             & 
		$\emptyset$				   & 
		$abcd$ 					   & 
		$\max$ 					   \\\midrule 
				
	    $abcd$ 					   & 
		$4-3$  		     		   & 
		$0-0$	    			   & 
		$\mathbf{1}$			   \\

		$abe$ 					   & 
		$3-3$  		     		   & 
		$3-0$   				   & 
		$3$      		   		   \\
		\bottomrule
	\end{tabular}	
	\caption{
		Relative Hamming surprise $s^\nu_\hamming(v, w)$, 
		$[\phi]=\{\emptyset,abcd\}$, 
		$[\mu] = \{\emptyset,abcd, abe\}$.
		The best interpretation is now $abcd$:
		the ranking induced by relative surprise depends on $\mu$,
		as well as $\phi$.
	}
	\label{tab:hamming-surprise-2}
\end{table}

\begin{example}\label{ex:hamming-surprise}
	Consider formulas $\phi$ and $\mu$ as in Example~\ref{ex:dalal-hmax},
	with $[\phi] = \{\emptyset, abcd\}$
	and $[\mu]=\{\emptyset,abcd, abe\}$.
	We have that $d_\hamming(\emptyset, \mu) = \min_{w\in[\mu]} d_\hamming(\emptyset,w)=0$, 
	and thus $s^\mu_{\hamming}(\emptyset,abcd) = d_\hamming(\emptyset, abcd)-d_{\hamming}(\emptyset,\mu)=4-0=4$.
	The surprise terms are depicted in Table \ref{tab:hamming-surprise-1}.
	We obtain, thus, that 
	$[\phi\circ^{s_{\hamming},\,\max}\mu] = [\phi\circ^{d_{\hamming},\,\max}\mu]=\{abe\}$.
	Consider, now, a formula $\nu$ with $[\nu]=\{abcd, abe\}$,
	with the surprise scores depicted in Table \ref{tab:hamming-surprise-2}.
	Note that in this case we obtain that 
	$[\phi\circ^{s_{\hamming},\,\max}\nu] =\{abcd\}$.
	Thus, in revision by $\mu$, $abe$ is chosen over $abcd$, 
	whereas in revision by $\nu$ the choice is reversed.
	Intuitively, when $\emptyset$ stops being a viable option,
	$abcd$ becomes more attractive than $abe$, 
	as the amount of surprise it would inflict, 
	from the standpoint of $\emptyset$,
	relative to $abe$, becomes smaller:
	considering the options, $abcd$ is not as extreme as $abe$. 
	In other words, for $\emptyset$ the two interpretations 
	$abcd$ and $abe$ are sufficiently alike to be considered almost equally risky: 
	the marginal surprise that $abcd$ carries over $abe$ is not big enough to be considered significant, 
	so that the final decision ends up choosing $abcd$ as carrying the least amount of risk. 
	By contrast, when $\emptyset$ is present as an option (see Table \ref{tab:hamming-surprise-1}) 
	the situation is markedly different, as the relative surprise of actually ending up 
	with $abcd$ or $abe$ becomes much more significant.
\end{example}

\noindent 
The type of scenario depicted in Example \ref{ex:hamming-surprise} is reminiscent of 
deviations from the principle of independence from irrelevant alternatives
signaled in the rational choice literature \cite{Sen1993-zh},
and immediately points toward a salient feature of the relative surprise operator
we have introduced:
it is not guaranteed to satisfy postulates $\mathrm{R}_{2}$, $\mathrm{R}_{5}$ and $\mathrm{R}_{6}$. 
Indeed, for $\phi$ and $\mu$ from Example \ref{ex:hamming-surprise}
we have that $[\phi\circ^{s_{\hamming},\,\max}\mu]=\{abe\}$, 
despite the fact that $[\phi\land\mu]=\{\emptyset,abcd\}$, 
which speaks to postulate $\mathrm{R}_{2}$.
Since $\phi\circ^{s_{\hamming},\,\max}\mu$ coincides, in this case, 
with $\phi\circ^{d_{\hamming},\,\max}\mu$, and 
$\circ^{d_{\hamming},\,\max}$ is already known not to satisfy postulate $\mathrm{R}_{2}$, 
this is perhaps not surprising, but similar reasoning shows that $\phi\circ^{s_{\hamming},\,\max}\mu$
does not satisfy postulates $\mathrm{R}_{7}$ and $\mathrm{R}_{8}$ either.
And $[\phi\circ^{s_{\hamming},\,\max}(\mu\land\nu)]=\{abcd\}$, 
despite the fact that $[(\phi\circ^{s_{\hamming},\,\max}\mu)\land\nu]=\{abe\}$, 
which speaks to postulates $\mathrm{R}_{5}$ and $\mathrm{R}_{6}$.
More to the point, the ranking on interpretations that is generated by the surprise measure $s+\hamming$
varies with $\mu$, to the extent that narrowing down the new information, as in Example \ref{ex:hamming-surprise},
can lead to inversions between the relative ranking of two interpretations.
At the same time, the ranking plainly depends on nothing more than $\phi$ and $\mu$, such that the 
result of revision is invariant to the syntax of the prior and new information. 
Additionally, $\circ^{s_{\hamming},\,\max}$ selects the result from the models of $\mu$, 
and is guaranteed to output \emph{something} as long as $\mu$ is consistent. 
We summarize these observations in the following proposition.

\begin{proposition}\label{prop:surprise-hamming-postulates}
	The operator $\circ^{s_{\hamming},\,\max}$ satisfies postulates 
	$\mathrm{R}_{1}$, 
	$\mathrm{R}_{3}$ and 
	$\mathrm{R}_{4}$,
	but not 
	$\mathrm{R}_{2}$, 
	$\mathrm{R}_{5}$,
	$\mathrm{R}_{6}$, 
	$\mathrm{R}_{7}$
	and 
	$\mathrm{R}_{8}$.
\end{proposition}

\noindent
One detail worth mentioning is that when $\phi$ is complete 
all operators presented so far coincide.

\begin{proposition}\label{prop:complete-equivalence}
	For any complete formula $\phi_v$,  
	$\phi \circ^{d_\hamming,\,\min}\mu\equiv \phi\circ^{d_{\hamming},\,\max}\equiv \phi\circ^{s_{\hamming},\,\max}\mu$, for any formula $\mu$.
\end{proposition}
\begin{proof}
	For complete $\phi_v$ it is only the relative ranking of interpretations with respect to $v$
	that matters, and this is the same for all three operators.
\end{proof}

\noindent
Proposition \ref{prop:surprise-hamming-postulates} shows that the $\circ^{s_{\hamming},\,\max}$ operator 
does not fit neatly into the standard revision framework.
However, since, we have argued, $\circ^{s_{\hamming},\,\max}$ formalizes an appealing intuition, 
it will be useful to unearth the general rules underpinning it: 
our goal, now, is to find a set of normative principles 
strong enough to characterize $\circ^{s_{\hamming},\,\max}$.
A set of such principles is offered in Section \ref{sec:surprise-hamming-max}, 
but, since $\circ^{s_{\hamming},\,\max}$ can be seen as a more involved min-max operator, 
we set the scene by first characterizing $\circ^{d_{\hamming},\,\max}$.
And to set the scene for $\circ^{d_{\hamming},\,\max}$, we first characterize 
the Dalal operator.

\section{Characterizing the Dalal Operator}\label{sec:dalal}
In this section we present a set of postulates that 
characterize the Dalal operator $\circ^{d_\hamming,\,\min}$.
Apart from being of independent interest,
this section presents, in the familiar setting of a known operator,
the main intuitions and techniques used in subsequent sections.
We start by introducing some additional new notions.

A \emph{renaming $r$ of $\Atoms$} is a bijective function $r\colon\Atoms\rightarrow\Atoms$.
If $\phi$ is a propositional formula, the \emph{renaming $r(\phi)$ of $\phi$}
is a formula $r(\phi)$ whose atoms are replaced according to $r$.
On the semantic side,
if $w$ is an interpretation and $r$ is a renaming of $\Atoms$,
the \emph{renaming $r(w)$ of $w$} is an interpretation obtained 
by replacing every atom $p$ in $w$ with $r(p)$.
If $\W$ is a set of interpretations, the \emph{renaming $r(\W)$ of $\W$} is 
defined as $r(\W)=\{r(w)\mid w\in\W\}$,
i.e., the set of interpretations whose elements are the renamed interpretations in $\W$.

A \emph{flip function $f\colon 2^{A}\times \L\rightarrow\L$} is 
a function that takes as input a set $v\subseteq A$ of atoms
(equivalently, $v$ can be thought of as an interpretation)
and a propositional formula $\phi$,
and returns a propositional formula $f_{v}(\phi)$ that is just like $\phi$ except 
that all the atoms from $v$ that appear in $\phi$ are flipped, i.e., replaced with their negations.
Overloading notation, a flip function applied to interpretations $v$ and $w$ returns 
an interpretation $f_{v}(w)$ in which all the atoms from $v$
that appear in $w$ are flipped, i.e., 
$f_v(w)=\{p\in A\mid p\in w~\text{and}~p\notin v,~\text{or}~p\in v~\text{and}~p\notin w\}$. 
It is straightforward to see that $f_v(w) = w\triangle v$.
If $\W$ is a set of interpretations, then $f_{v}(\W) = \{f_{v}(w)\mid w\in\W\}$,
i.e., the set of interpretations obtained by flipping every atom in $v$.

\begin{example}\label{ex:flipping-renaming}	
	For the set $A = \{a,b,c\}$ of atoms, 
	consider a formula $\phi = a\land\lnot c$, 
	with $[\phi]=\{a,ab\}$,
	and a renaming $r$ such that $r(a)=b$, $r(b)=c$ and $r(c)=a$. 
	We obtain that $r(\phi) = r(a)\land\lnot r(c) = b\land \lnot a$, 
	with $[r(\phi)]=\{b,bc\} = \{r(a), r(ab)\}$.
	Flipping atoms $b$ and $c$, we have that 
	$f_{bc}(\phi) = a\land \lnot (\lnot c)$, 
	with $[f_{bc}(\phi)] = \{abc, ac\}$.
	Note that $[f_{bc}(\phi)] = \{f_{bc}(a), f_{bc}(ab)\} = \{a\triangle bc, ab \triangle bc\}$.
\end{example}

\noindent
In Example \ref{ex:flipping-renaming} it holds that:
($i$) $[r(\phi)] = r([\phi])$,
($ii$) $[f_{w}(\phi)] = f_{w}([\phi])$
and ($iii$) $[f_{w}(\phi)] = \{v\triangle w\mid v\in[\phi]\}$,
and we note here that all these equalities hold generally 
(for ($ii$) see, for instance, Exercise 2.28 in \cite{Goldrei2005}).
Their relevance will become apparent shortly.

To characterize the Dalal operator $\circ^{d_\hamming,\,\min}$ 
we introduce a set of new postulates, starting with 
\emph{Neutrality} $\mathrm{R}_{\Neutrality}$:

\begin{description}
	\item[($\mathrm{R}_{\Neutrality}$)] If $\phi$ is complete, 
		then $r(\phi\circ\mu)\equiv r(\phi)\circ r(\mu)$.
\end{description}

\noindent
Postulate $\mathrm{R}_{\Neutrality}$ states that revision is invariant 
under renaming atoms and hence neutral in that the specific labels 
for the atoms do not matter towards the final result. 
This postulate is inspired by similar ideas in social choice 
and has appeared before in belief change contexts \cite{HerzigR1999,MarquisS2014,HaretW2019}.

The next postulate concerns the effect of 
flipping the same atoms in both $\phi$ and $\mu$, and is called, appropriately,
the \emph{Flipping} postulate $\mathrm{R}_{\Flipping}$:

\begin{description}
	\item[($\mathrm{R}_{\Flipping}$)] If $\phi$ is complete, then $f_{v}(\phi\circ\mu) = f_{v}(\phi)\circ f_{v}(\mu)$.
\end{description}

\noindent
An additional constraint, the \emph{Addition} postulate $\mathrm{R}_{\Addition}$, 
is obtained by considering the effect of adding new atoms that affect the 
standing of one interpretation, and is meant to apply to any formulas $\phi$ and $\mu$ and 
set $x$ of new atoms, i.e., such that none of the atoms in $x$ appears in either $\phi$ or $\mu$:
\begin{description}
	\item[($\mathrm{R}_{\Addition}$)] If  $\phi$ is complete and $(\phi\circ\mu_{w_1,w_2})\land\mu_{w_1}$ is consistent, then $\phi\circ\mu_{w_1,w_2 \cup x}\equiv \mu_{w_1}$.
\end{description}

\noindent
Postulate $\mathrm{R}_{\Addition}$ is best understood through a choice perspective:
if $w_1$ is chosen by $\phi$ over $w_2$ when the choice is $[\mu_{w_1,w_2}]=\{w_1,w_2\}$, 
then adding extra new atoms $x$ to $w_2$, 
(and, thereby, increasing the distance to $\phi$) 
ensures that $w_{2}\cup x$ is not chosen when the choice is $[\mu_{w_1,w_2\cup x}]=\{w_1,w_2\cup x\}$.
In all of these postulates the prior belief $\phi$ is assumed to be complete: 
this is not essential for the characterization of the Dalal operator, 
but makes life easier in the characterization of the surprise minimization operator, in Section \ref{sec:surprise-hamming-max}.

The next postulate involves a mix of flips and we ease into it by introducing an intermediary notion.
The \emph{best-of-best formula $\beta_{\phi, \mu}$ with respect to $\phi$ and $\mu$}
is defined as:
$$
	\beta_{\phi,\mu} = \epsilon\circ\Big(\bigvee_{v\in[\phi]} f_{v}(\mu)\Big),
$$
i.e., as the result of revising the null formula $\epsilon$ 
(recall that $[\epsilon]=\{\emptyset\}$) by a disjunction made up 
of multiple versions of $\mu$, where each such version is obtained by flipping 
the atoms in a model $v$ of $\phi$.
Intuitively, the intention is to recreate the table of Hamming distances (e.g., Table \ref{tab:dalal-hmax}) without using numbers: 
recall that $[f_v(\mu)] = f_v([\mu])$ and $f_v(w) = w\triangle v$
and thus, semantically, we have that 
$[\bigvee_{v\in[\phi]} f_{v}(\mu)] = \{w_i\triangle v_j\mid w_i\in[\mu], v_j\in[\phi]\}$.
In other words, we are creating a scenario in which $\epsilon$ has to choose between interpretations obtained 
as the symmetric difference of the elements of $[\phi]$ and $[\mu]$. 
The result we are working towards, yet to be proven, is that an element of $[\bigvee_{v\in[\phi]} f_{v}(\mu)]$ chosen by $\epsilon$, 
i.e., an interpretation $w_i\triangle v_j\in[\beta_{\phi,\mu}]$,
corresponds to an interpretation $w_i\in[\mu]$ that minimizes the overall Hamming distance to $\phi$, 
and is thus among the best of the best interpretations in this revision scenario. 
The role of the \emph{Best-of-Best} postulate $\mathrm{R}_{\BoB}$, then, is to recover the models of $\mu$ from 
the models of $\beta_{\phi,\mu}$:

\begin{description}
	\item[($\mathrm{R}_{\BoB}$)] $\phi\circ\mu\equiv \bigg(\bigvee_{v\in[\phi]} f_{v}(\beta_{\phi,\mu})\bigg)\land\mu$.
\end{description}

\noindent
Postulate $\mathrm{R}_{\BoB}$ stipulates that the result of revising $\phi$ by $\mu$
consists of those interpretations of $\mu$ that come out of flipping $\beta_{\phi,\mu}$
by each model of $\phi$, in this way reversing the initial flips that delivered the 
revision formula posed to $\epsilon$.

What is the significance of the null formula $\epsilon$ in $\beta_{\phi,\mu}$?
We want to reduce arbitrary revision tasks to a common denominator, a base case in which 
the result of revision can be decided without explicit appeal to distances (i.e., numbers), 
and only by appeal to desirable normative principles, such as the postulates laid out above.
The case when the prior belief is $\epsilon$ turns out to be well suited for this task, 
since, as
we show next, postulates 
$\mathrm{R}_{1}$, $\mathrm{R}_{3}$-$\mathrm{R}_{6}$, $\mathrm{R}_{\Neutrality}$ and $\mathrm{R}_{\Addition}$ 
guarantee that $\epsilon$ always selects the interpretations with minimal cardinality.

\begin{lemma}\label{lem:empty-set-chooses-well}
	If a revision operator $\circ$ satisfies postulates 
	$\mathrm{R}_{1}$, 
	$\mathrm{R}_{3}$-$\mathrm{R}_{6}$, 
	$\mathrm{R}_{\Neutrality}$ and
	$\mathrm{R}_{\Addition}$,
	then, for any formula $\mu$, 
	it holds that $[\epsilon\circ\mu]=\mathrm{argmin}_{w\in[\mu]} |w|$.
\end{lemma}
\begin{proof}
	(``$\subseteq$'')
	Suppose, first, that $w_{1}\in[\epsilon\circ \mu]$ and there is $w_2\in[\mu]$ 
	such that $|w_{1}|>|w_{2}|$.
	Using postulate $\mathrm{R}_{5}$ we obtain that $w_1\in[\phi\circ\mu_{w_1,w_2}]$.
	We now show that this leads to a contradiction, 
	and we do this using the Neutrality postulate $\mathrm{R}_{\Neutrality}$:
	however, we would like to apply $\mathrm{R}_{\Neutrality}$ to interpretations of equal size.
	Towards this, take a set $x$ of new atoms (i.e., that do not occur in either $\phi$ or $\mu$), 
	with $|x| = |w_1|-|w_2|$, and add $x$ to $w_2$ to form $w_{2}' = w_{2}\cup x$.
	We have that $|w'_{2}| = |w_{2}|+(|w_{1}|-|w_{2}|) = |w_{1}|$, 
	i.e., $w_1$ and $w'_2$ are of the same size,
	which implies that $|w_1\setminus w'_2| = |w'_2\setminus w_1|$.
	Applying the addition postulate $\mathrm{R}_{\Addition}$, 
	we obtain that $w'_2\notin[\epsilon\circ\mu_{w_1,w'2}]$. 
	
	Consider, now, a renaming $r$ that swaps atoms 
	in $w_1\setminus w'_2$ with atoms in $w'_2\setminus w_1$, 
	made possible by the fact that $w_1\setminus w'_2$ and $w'_2\setminus w_1$
	are of the same size. 
	This implies that $r(w_1) = w'_2$ and $r(w'_2) = w_1$ and thus 
	$r([\mu_{w_1,w'_2}]) = r(\{w_1,w'_2\}) = \{r(w_1), r(w'_2)\} = \{w'_2,w_1\} = [\mu_{w_1,w'_2}]$.
	Applying the Neutrality postulate $\mathrm{R}_{\Neutrality}$ to 
	$\epsilon\circ\mu_{w_1,w'_2}$ with the renaming $r$ thus defined, 
	and, keeping in mind that $[r(\epsilon)] = [\epsilon]$,
	and thus that $r(\epsilon) \equiv \epsilon$, 
	we obtain that:
	\begin{align*}
		\{w_1\} & = [\epsilon\circ\mu_{w_1,w'_2}] &~\text{by assumption and}~\Addition\\ 
		        & = [r(\epsilon)\circ r(\mu_{w_1,w'_2})] &~\text{by def. of}~r~\text{and}~\mathrm{R}_{4}\\ 
				& = [r(\epsilon\circ \mu_{w_1,w'_2}))] &~\text{by}~\Neutrality\\ 
				& = r([\epsilon\circ \mu_{w_1,w'_2}]) &~\text{property of}~r\\ 
				& = r(\{w_1\}) & ~\text{by assumption}\\ 
				& = \{w'_2\}.
	\end{align*}
	This implies that $w_1 = w'_2$ but, since $w'_2$ contains a non-negative number of 
	atoms that do not appear in $w_1$, this is a contradiction.

	(``$\supseteq$'')
	For the opposite direction, suppose that $w_1\in \mathrm{argmin}_{w\in[\mu]}|w|$ 
	but $w_1\notin[\epsilon\circ\mu]$.
	Using postulates $\mathrm{R}_{1}$ and $\mathrm{R}_{3}$ we have that 
	there is $w_2\in[\phi\circ\mu]$ and, with postulate $\mathrm{R}_{6}$
	we obtain that $[\epsilon\circ \mu_{w_1,w_2}] = \{w_2\}$. 
	Since $|w_1|\le|w_2|$
	we add to $w_1$
	a set $x$ of new atoms, where $|x| = |w_2|-|w_1|$, 
	and denote $w'_1 = w_1\cup x$. Applying $\mathrm{R}_{\Addition}$
	we obtain that $[\epsilon\circ \mu_{w'_1,w_2}]=\{w_2\}$
	and, using a renaming $r$ defined, as in the previous direction, 
	such that $r(w_2)=w'_1$ and $r(w'_1)=w_2$,
	and applying $\mathrm{R}_{\Neutrality}$ to $r$ and $\epsilon\circ\mu_{w'_1,w_2}$,
	we obtain that $[\epsilon\circ \mu_{w'_1,w_2}]=\{w'_1\}$, 
	leading to a contradiction.
\end{proof}

\noindent
Lemma \ref{lem:empty-set-chooses-well} shows that, in the very particular case in which 
the prior belief is $\epsilon$, we can ensure that the result of revision coincides with the 
result delivered by the Dalal operator. 
The next move consists in using 
the Flipping postulate $\mathrm{R}_{\Flipping}$ to extend this fact to complete formulas.

\begin{lemma}\label{lem:complete-chooses-well}
	If a revision operator $\circ$ satisfies postulates 
	$\mathrm{R}_{1}$, 
	$\mathrm{R}_{3}$-$\mathrm{R}_{6}$, 
	$\mathrm{R}_{\Neutrality}$,
	$\mathrm{R}_{\Addition}$ and 
	$\mathrm{R}_{\Flipping}$,
	then, for any formula $\mu$
	and complete formula $\phi_v$,
	it holds that 
	$[\phi_v\circ\mu]=\mathrm{argmin}_{w\in[\mu]}d_\hamming(v,w)$.
\end{lemma}
\begin{proof}
	By postulate $\mathrm{R}_{\Flipping}$ it holds that 
	$f_v(\phi_v\circ\mu)\equiv f_v(\phi_v)\circ f_v(\mu)$.
	Note, now, that $[f_v(\phi_v)]=\{v\triangle v\}=\{\emptyset\}$, 
	and thus $f_v(\phi_v)\equiv\epsilon$, 
	while $[f_v(\mu)]=\{w\triangle v\mid w\in[\mu]\}$.
	By Lemma \ref{lem:empty-set-chooses-well}, it holds that 
	$[\epsilon\circ f_v(\mu)]=\min_{w\triangle v\in[f_v(\mu)]}|w\triangle v|$ 
	and, since $d_\hamming(v,w)=|w\triangle v|$, we derive the conclusion.
\end{proof}

\noindent
Lemma \ref{lem:complete-chooses-well} shows that it is not just 
the formula $\epsilon$ that makes choices consistent with the 
Dalal operator, but any complete formula $\phi_{v}$.
The intuition driving Lemma \ref{lem:complete-chooses-well}
is that the situation where $v$ chooses between $w_1$ and $w_2$ 
is equivalent, through the Flipping postulate $\mathrm{R}_{\Flipping}$,
to a scenario where $\emptyset$ chooses between $w_1\triangle v$ and $w_2\triangle v$: 
and we know that in this situation postulates 
$\mathrm{R}_{\Neutrality}$ and $\mathrm{R}_{\Addition}$ 
guide $\emptyset$ to choose the interpretation $w_i\triangle v$ of minimal cardinality, 
which corresponds to $w_i$ being at minimal Hamming distance to $v$.

The next step involves pushing this intuition even further,
to the case of any propositional formula $\phi$. 
As anticipated, the Best-of-Best postulate $\mathrm{R}_{\BoB}$
is the postulate that facilitates this move, and the proof goes 
through the intermediary obervation that the best-of-best formula 
$\beta_{\phi,\mu}$ selects interpretations corresponding to the desired redult.


\begin{lemma}\label{lem:any-chooses-well}
	If $\circ$ is a revision operator that satisfies postulates 
	$\mathrm{R}_{1}$, 
	$\mathrm{R}_{3}$-$\mathrm{R}_{6}$ 
	$\mathrm{R}_{4}$, 
	$\mathrm{R}_{\Neutrality}$,
	$\mathrm{R}_{\Addition}$ and
	$\mathrm{R}_{\Flipping}$
	then, for any formulas $\phi$ and $\mu$
	and interpretations $w$ and $v$,
	it holds that
	$w\triangle v\in[\beta_{\phi,\mu}]$
	if and only if
	$w\in \mathrm{argmin}_{w\in[\mu]}\min_{v\in[\phi]}d_\hamming(v, w)$.
\end{lemma}
\begin{proof}
	By Lemma \ref{lem:empty-set-chooses-well}, $[\beta_{\phi,\mu}]$ chooses 
	exactly those interpretations $w_i\triangle v_j$, for $w_i\in[\mu]$ and $v_j\in[\phi]$,
	that are of minimal cardinality. Since $|w_i\triangle v_j|=d_\hamming(w_i, v_j)$,
	the conclusion follows immediately.
\end{proof}

\noindent
By Lemma \ref{lem:any-chooses-well}, the result of the Dalal operator $\circ^{d_\hamming,\,\min}$
applied to $\phi$ and $\mu$ consists of those interpretations $w\in[\mu]$ such 
that $w\triangle v\in[\beta_{\phi,\mu}]$, for some $v\in[\phi]$.
The Best-of-Best postulate $\mathrm{R}_{\BoB}$ instructs us that these are exactly the 
models of $\mu$ that should be chosen by an operator $\circ$,
and provides the final piece in the sought after characterization. 

\begin{theorem}\label{th:characterization-dalal}
	A revision operator $\circ$ satisfies postulates 
	$\mathrm{R}_{1}$, 
	$\mathrm{R}_{3}$-$\mathrm{R}_{6}$,
	$\mathrm{R}_{\Neutrality}$,
	$\mathrm{R}_{\Addition}$,
	$\mathrm{R}_{\Flipping}$
	and $\mathrm{R}_{\BoB}$
	if and only if $\circ \equiv \circ^{d_\hamming,\,\min}$.
\end{theorem}
\begin{proof}
	For one direction, 
	we take as known that the Dalal operator $\circ^{d_\hamming,\,\min}$ 
	satisfies postulates
	$\mathrm{R}_{1}$, 
	$\mathrm{R}_{3-6}$ \cite{KatsunoM1992}
	and $\mathrm{R}_{\Neutrality}$ \cite{HaretW2019}.
	For postulate $\mathrm{R}_{\Neutrality}$,
	given Lemma \ref{lem:any-chooses-well}, satisfaction of postulates 	
	$\mathrm{R}_{\Neutrality}$,
	$\mathrm{R}_{\Addition}$,
	$\mathrm{R}_{\Flipping}$
	and $\mathrm{R}_{\BoB}$ follows straightforwardly.

	For the other direction, 
	we have to show that if $\circ$ satisfies all the stated postulates,
	then $[\phi\circ\mu]=\mathrm{argmin}_{w\in[\mu]}\min_{v\in[\phi]}d_\hamming(v, w)$,
	for any formulas $\phi$ and $\mu$.
	Lemma \ref{lem:any-chooses-well} already gives us that $\beta_{\phi,\mu}$
	selects those interpretations $w_i\triangle v_j$ for which $d_\hamming(w_i, v_j)$
	is minimal among the set $\{w\triangle v\mid w\in[\mu], v\in[\phi]\}$ of 
	symmetric differences between models of $\phi$ and of $\mu$.
	This means that if $w_i\in\mathrm{argmin}_{w\in[\mu]}\min_{v\in[\phi]}d_\hamming(v, w)$, 
	then $w_i\triangle v_j\in[\beta_{\phi,\mu}]$, for some $v_j\in[\phi]$, 
	and hence $(w_i\triangle v_j)\triangle v_j=w_i\in [f_{v_j}(\beta_{\phi,\mu})]$, 
	i.e., if $w_i$ is selected by the Dalal operator then it shows up in  
	$[(\bigvee_{v\in[\phi]} f_{v}(\beta_{\phi,\mu}))\land\mu]$.
	Conversely, suppose there is an interpretation $w_i\in[(\bigvee_{v\in[\phi]} f_{v}(\beta_{\phi,\mu}))\land\mu]$
	that is not at minimal distance to $\phi$. 
	This means that $w_i = (w_j\triangle v_k)\triangle v_l$, where $w_j\in[\mu]$
	corresponds to a model of $\mu$ that is at minimal Hamming distance to $\phi$ and $v_k,v_l\in[\phi]$.
	We infer from this that $w_i\triangle v_l = ((w_j\triangle v_k)\triangle v_l)\triangle v_l = w_j\triangle v_k$,
	and thus $|w_i\triangle v_l| = |w_j\triangle v_k|$. But this contradicts the assumed minimality of $w_j\triangle v_k$.
\end{proof}

\noindent 
Note that postulate $\mathrm{R}_{2}$ is not present in Theorem \ref{th:characterization-dalal}, 
even though the Dalal operator satisfies it, as it follows from the other postulates.

Theorem \ref{th:characterization-dalal} can be read not just as a characterization of the Dalal operator, 
but also as a recipe, or a step-by-step argument, 
for constructing $\phi\circ\mu$ from a set of simpler problems, 
in a srquence of steps guided
by the transformations inherent in postulates 
$\mathrm{R}_{\Neutrality}$,
$\mathrm{R}_{\Addition}$,
$\mathrm{R}_{\Flipping}$
and $\mathrm{R}_{\BoB}$.
The form such an argument could take is illustrated in the following example.

\begin{figure}
	\centering
	\begin{tikzpicture}
		\node at (0,0.7)(a){$a$};
		\node at (1.5,1.2)(ac){$ac$};
		\node at (1.5, 0.2)(abc){$abc$};
		\path 
			(a)edge node[above]{\tiny $1$}(ac)
			(a)edge node[below]{\tiny $2$}(abc);

		\node at (0,-0.7)(b){$b$};
		\node at (1.5,-0.2)(ac2){$ac$};
		\node at (1.5, -1.2)(abc2){$abc$};
		\path 
			(b)edge node[above]{\tiny $3$}(ac2)
			(b)edge node[below]{\tiny $2$}(abc2);

		\node at (6,0)(e){$\emptyset$};

		\node at (4.5, 1.2)(a){$\bm{c}$};
		\node at (4.5, 0.2)(bc){$bc$};
		\node at (4.5, -0.2)(abc3){$abc$};
		\node at (4.5, -1.2)(ac3){$ac$};
		\path 
			(e)edge node[above]{\tiny $1$}(a)
			(e)edge node[above]{\tiny $2$}(bc)
			(e)edge node[below]{\tiny $3$}(abc3)
			(e)edge node[below]{\tiny $2$}(ac3)
			;
		
		\path[-latex,  dotted]
			(ac)edge node[above]{\tiny flip $a$}(a)
			(abc)edge node[above]{\tiny flip $a$}(bc)
			(ac2)edge node[below]{\tiny flip $b$}(abc3)
			(abc2)edge node[below]{\tiny flip $b$}(ac3)
			;
	\end{tikzpicture}
	\caption{
		By flipping the atoms of $v\in[\phi]$ in a model $w$ of $\mu$ 
		we get an interpretation $w\triangle v$ whose size corresponds to the 
		Hamming distance between $v$ and $w$, 
		i.e., $|v\triangle w|=d_\hamming(v, w)=d_\hamming(\emptyset, v\triangle w)=d_\hamming(\emptyset, f_v(w))$.
		In this way, flipped models that get chosen by $\epsilon$ 
		corresponds to models of $\mu$ that minimize overall Hamming distance to $\phi$.
	}
	\label{fig:dalal-postulates}
\end{figure}
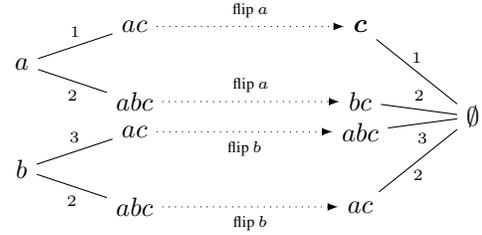
\begin{example}\label{ex:dalal-postulates}
	Consider formulas $[\phi]=\{a,b\}$ and $[\mu]=\{ac,abc\}$
	and note, first, that $[\phi \circ^{d_\hamming,\,\min} \mu]=\{ac\}$, 
	as $ac$ minimizes overall distance to $\phi$ \emph{via} $d_\hamming(a, ac)=1$.
	Assume, however, that we are given a revision operator $\circ$ 
	that is not defined using distances, 
	but is presented only as satisfying postulates
	$\mathrm{R}_{1}$, 
	$\mathrm{R}_{3-6}$,
	$\mathrm{R}_{\Neutrality}$,
	$\mathrm{R}_{\Addition}$,
	$\mathrm{R}_{\Flipping}$
	and $\mathrm{R}_{\BoB}$.
	An agent revising according to $\circ$ can use the postulates to work its way toward 
	$[\phi \circ^{d_\hamming,\,\min} \mu]$ without knowing anything about distances.
	This can be done by, first, splitting the problem into two revision problems, 
	one for each model of $\phi$: $\phi_a\circ\mu$ and $\phi_{b}\circ\mu$, 
	where $[\phi_a]=\{a\}$ and $[\phi_{b}]=\{b\}$.
	The next step consists in reducing both problems to the common denominator
	of revising with prior belief $\epsilon$, where $[\epsilon]=\{\emptyset\}$.
	This is done by flipping $a$ and $b$, respectively, in the two problems, to 
	obtain the revision scenarios $\epsilon\circ f_{a}(\mu)$ and $\epsilon\circ f_{b}(\mu)$, 
	with $[f_a(\mu)]=\{f_a(ac),f_{a}(abc)\}=\{ac\triangle a, abc \triangle a\}=\{c, bc\}$ and, 
	likewise, $[f_{b}(\mu)]=\{abc,ac\}$ (see Figure \ref{fig:dalal-postulates}).
	This move preserves Hamming distances in a crucial way: 
	to take one instance, $d_\hamming(a, ac)=1$, 
	where $a\in[\phi]$ and $ac\in[\mu]$,
	coincides with the Hamming distance between $\emptyset$ and $f_a(ac)=c$,
	and this distance coincides with the number of atoms in $f_a(ac)=c$.
	The operator $\circ$, of course, knows nothing of this: 
	it performs these transformations solely because 
	postulate $\mathrm{R}_{\BoB}$ warrants them.
	Thus, in the next step $\epsilon$ chooses among the models obtained from the successive flips of $\mu$,
	i.e., it solves the revision problem $\epsilon\circ(f_a(\mu)\lor f_b(\mu))$.
	Postulates 
	$\mathrm{R}_{1}$,
	$\mathrm{R}_{3}$-$\mathrm{R}_{6}$,
	$\mathrm{R}_{\Neutrality}$
	and 
	$\mathrm{R}_{\Addition}$,
	\emph{via} the argument in Lemma \ref{lem:empty-set-chooses-well},
	dictate that $\epsilon$ chooses the interpretation of minimal cardinality,
	such that $[\beta_{\phi,\mu}] = [\epsilon\circ(f_a(\mu)\lor f_{b}(\mu))]=\{c\}$.
	The result obtained, i.e., interpretation $c$, is the result of flipping the atom $a$ in 
	the interpretation $ac\in[\mu]$:
	to recover $ac$ from $c$, we `reverse' the original flips: one flip 
	by $a$ and one by $b$, to get 
	$[f_a(\beta_{\phi,\mu})\lor f_{b}(\beta_{\phi,\mu})]=\{ac,bc\}$.
	By postulate $\mathrm{R}_{\BoB}$, 
	we have that $[\phi\circ\mu] = [f_a(\beta_{\phi,\mu})\lor f_{b}(\beta_{\phi,\mu})\land\mu]=\{ac\}$,
	i.e., exactly the result produced by the Dalal operator $\circ^{d_\hamming,\,\min}$.
\end{example}

\section{Characterizing the Hamming Distance Min-Max Operator}\label{sec:distance-hamming-max}
The postulates put forward in Section \ref{sec:dalal} for characterizing the Dalal operator
prove their worth in an additional sense,
as they can be put to use, with minimal modifications, in characterizing 
the Hamming distance min-max operator $\circ^{d_{\hamming},\,\max}$.
This is the topic of the current section.
 
Of the newly proposed postulates, the Neutrality, Addition and Flipping postulates 
($\mathrm{R}_{\Neutrality}$, $\mathrm{R}_{\Addition}$ and $\mathrm{R}_{\Flipping}$, respectively)
can be used as stated in Section \ref{sec:dalal}, while the Best-of-Best postulate $\mathrm{R}_{\BoB}$
has to be modified. 
Intuitively, this makes sense: postulates $\mathrm{R}_{\Neutrality}$, $\mathrm{R}_{\Addition}$ and $\mathrm{R}_{\Flipping}$
are used in regulating what happens when the prior information is a complete formula $\phi_v$
(alternatively, for what happens in the ranking that corresponds to the $v$-column in the table of distances, e.g., Table \ref{tab:dalal-hmax}), 
in which case, as per Proposition \ref{prop:complete-equivalence}, all operators presented here coincide,
whereas postulate $\mathrm{R}_{\BoB}$ instructs us how to choose when the prior information consists of more than one model 
(alternatively, across different columns of the table of distances).
Correspondingly, postulate $\mathrm{R}_{\BoB}$ encodes the constraint that revision should pick the best of the best models across all of the $\phi_v$'s, 
for $v\in[\phi]$, but this is not the rule that defines operator $\circ^{d_{\hamming},\,\max}$.
For $\circ^{d_{\hamming},\,\max}$ we need a principle 
that mandates picking the best of the worst models across the $\phi_v$'s.
The key fact allowing us to do this relies on a certain duality specific to the Hamming distance
that will guide us in designing an appropriate postulate 
for $\circ^{d_{\hamming},\,\max}$, 
and which is summarized in the following result.
Recall that $\Atoms$ is the set of all atoms.

\begin{lemma}\label{lem:hamming-duality}
	If $v$ and $w$ are interpretations and $|\Atoms|=n$, 
	then $d_\hamming(v, w) = n-d_\hamming(A\setminus v, w)$.
\end{lemma}

\noindent
Intuitively, Lemma \ref{lem:hamming-duality} implies that the further away $w$ is from $v$
(in terms of Hamming distance),
the closer $w$ is to $\Atoms \setminus v$.
In particular, we can infer that:
\begin{align}\label{eq:hamming-duality}
	d_\hamming(v, w) & = d_\hamming(\emptyset, |v\triangle w|)\nonumber\\
					 & = d_\hamming(\emptyset, f_v(w))\nonumber\\ 
					 & = n - d_\hamming(A, f_v(w)).
\end{align}

\noindent 
Hence, $w\in[\mu]$ is among the models of $\mu$ at maximal Hamming distance to $v$
if and only if $f_{v}(w)$ is, among the models of $f_{v}(\mu)$, 
the closest to $A$, 
or, more intuitively,
the worst model of $\mu$ according to $v$ is the best model of $f_v(\mu)$ according to $\alpha$, 
where $[\alpha]=A$.
We can thus define the 
\emph{best-of-worst formula $\gamma_{\phi, \mu}$ with respect to $\phi$ and $\mu$} as:
$$
	\gamma_{\phi,\mu} = \epsilon\circ\bigg(\bigvee_{v\in[\phi]}\Big(\alpha\circ f_{v}(\mu)\Big)\bigg),
$$
i.e., as the result of revising the null formula $\epsilon$ 
by a disjunction made up of the results obtained from a sequence of 
revisions of the full formula $\alpha$.
In this sequence $\alpha$ is revised, in turn, by $f_v(\mu)$, for every model $v\in[\phi]$.

Thus, similarly as for $\beta_{\phi,\mu}$ from Section \ref{sec:dalal}, 
$\gamma_{\phi,\mu}$ simulates the process of going through the table of Hamming distances 
(e.g., Table \ref{tab:dalal-hmax}), 
except that in this case we are interested in 
($i$) selecting the worst elements according to each $\phi_v$, for $v\in[\phi]$, 
an operation reflected by the revision $\alpha\circ f_v(\mu)$,
and 
($ii$) selecting the best among these worst elements, 
an operation reflected by submitting the results obtained previously 
to $\epsilon$ for an additional round of revision.
A bespoke postulate, called the \emph{Best-of-Worst} postulate $\mathrm{R}_{\BoW}$, 
recovers the models of $\mu$ from the models of $\gamma_{\phi,\mu}$:

\begin{description}
	\item[($\mathrm{R}_{\BoW}$)] $\phi\circ\mu\equiv\bigg(\bigvee_{v\in[\phi]} f_{v}(\gamma_{\phi,\mu})\bigg)\land\mu$.
\end{description}

\noindent
Postulate $\mathrm{R}_{\BoW}$ stipulates that the result of revising $\phi$ by $\mu$
consists of those models of $\mu$ that come out of flipping $\gamma_{\phi,\mu}$
by each model of $\phi$, in this way reversing the initial flips that delivered the 
revision formula posed to $\epsilon$.

The proof that the postulates put forward 
actually characterize the $\circ^{d_{\hamming},\,\max}$ operator
hinges on $\gamma_{\phi,\mu}$ selecting interpretations
corresponding to models $w$ of $\mu$ that minimize maximal Hamming distance to $\phi$.

\begin{lemma}\label{lem:dmax-chooses-well}
	If $\circ$ is a revision operator that satisfies postulates 
	$\mathrm{R}_{1}$, 
	$\mathrm{R}_{3}$-$\mathrm{R}_{6}$,
	$\mathrm{R}_{\Neutrality}$,
	$\mathrm{R}_{\Addition}$,
	$\mathrm{R}_{\Flipping}$
	and $\mathrm{R}_{\BoW}$,
	then, for any formulas $\phi$ and $\mu$
	and interpretations $w$ and $v$,
	it holds that $w\triangle v \in [\gamma_{\phi,\mu}]$ 
	if and only if 
	$w\in\mathrm{argmin}_{w\in[\mu]}\max_{v\in[\phi]}d_\hamming(v, w)$.
\end{lemma}
\begin{proof}
	Using postulates 
	$\mathrm{R}_{\Neutrality}$,
	$\mathrm{R}_{\Addition}$ and
	$\mathrm{R}_{\Flipping}$
	we can prove that $\alpha$ selects the models
	of $\mu$ that minimize Hamming distance to $A$, 
	in a way completely analogous to 
	Lemmas \ref{lem:empty-set-chooses-well} and Lemma \ref{lem:complete-chooses-well}.
	Thus, using Equality \ref{eq:hamming-duality},
	$\alpha\circ f_v(\mu)$ selects interpretations $w\triangle v$
	such that $d_\hamming(v, w)=\max_{w'\in[\mu]}d_\hamming(v, w')$.
	Then, using Lemma \ref{lem:empty-set-chooses-well},
	we obtain that $\gamma_{\phi,\mu}$ selects interpretations 
	$w\triangle v$ where $w$ minimizes max-distance to $\phi$.
\end{proof}

\noindent 
With Lemma \ref{lem:dmax-chooses-well} the characterization of $\circ^{d_{\hamming},\,\max}$ follows immediately.

\begin{theorem}\label{th:characterization-dhmax}
	If $\circ$ is a revision operator, 
	then $\circ$ satisfies postulates 
	$\mathrm{R}_{1}$, 
	$\mathrm{R}_{3}$-$\mathrm{R}_{6}$,
	$\mathrm{R}_{\Neutrality}$,
	$\mathrm{R}_{\Addition}$,
	$\mathrm{R}_{\Flipping}$
	and $\mathrm{R}_{\BoW}$
	iff 
	$\circ \equiv \circ^{d_\hamming,\,\max}$.
\end{theorem}
\noindent 
The proof is similar, in its essentials, to the proof of Theorem~\ref{th:characterization-dalal}
and is therefore omitted. The following example, however, illustrates 
how the mechanism works on a concrete case.

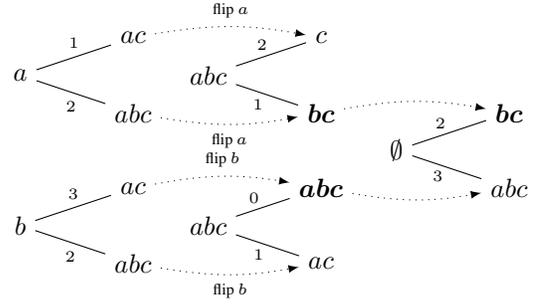
\begin{figure}
	\centering
	\begin{tikzpicture}
		\node at (0,1)(a){$a$};
		\node at (1.5,1.5)(ac){$ac$};
		\node at (1.5, 0.5)(abc){$abc$};
		\path 
			(a)edge node[above]{\tiny $1$}(ac)
			(a)edge node[below]{\tiny $2$}(abc);

		\node at (0,-1)(b){$b$};
		\node at (1.5,-0.5)(ac2){$ac$};
		\node at (1.5, -1.5)(abc2){$abc$};
		\path 
			(b)edge node[above]{\tiny $3$}(ac2)
			(b)edge node[below]{\tiny $2$}(abc2);

		\node at (4, 1.5)(a){$c$};
		\node at (4, 0.5)(bc){$\bm{bc}$};
		\node at (4, -0.5)(abc3){$\bm{abc}$};
		\node at (4, -1.5)(ac3){$ac$};

		\path[-latex,  dotted]
			(ac)edge[bend left=10] node[above]{\tiny flip $a$}(a)
			(abc)edge[bend right=10] node[below]{\tiny flip $a$}(bc)
			(ac2)edge[bend left=10] node[above]{\tiny flip $b$}(abc3)
			(abc2)edge[bend right=10] node[below]{\tiny flip $b$}(ac3)
		;
		\node at (2.5,1)(A1){$abc$};
		\node at (2.5,-1)(A2){$abc$};

		\path
			(a)edge node[above left=-2]{\tiny $2$}(A1)
			(bc)edge node[below left=-2]{\tiny $1$}(A1)
			(abc3)edge node[above left=-2]{\tiny $0$}(A2)
			(ac3)edge node[below left=-2]{\tiny $1$}(A2)
			;

		\node at (6.5, 0.5)(bc2){$\bm{bc}$};
		\node at (6.5, -0.5)(abc4){$abc$};
		\node at (5,0)(e){$\emptyset$};
		\path 
			(e)edge node[above left=-2]{\tiny $2$}(bc2)
			(e)edge node[below left=-2]{\tiny $3$}(abc4)
			;
		\path[-latex,  dotted]
			(bc)edge[bend left=10](bc2)
			(abc3)edge[bend right=10](abc4)
			;
	\end{tikzpicture}
	\caption{
		To get the best of the worst models of $\mu$ according to $a$ and $b$
		we got through two rounds of revision: first, flip $\mu$ by $a$ and by $b$.
		The results of $\alpha\circ f_a(\mu)$ and $\beta_\circ f_b(\mu)$ correspond 
		to the models of $\mu$ at maximal distance to $a$ and $b$, respectively. 
		This result is further refined by passing it to $\epsilon$ for revision.
	}
	\label{fig:hmax-postulates}
\end{figure}
\begin{example}\label{ex:hmax-postulates}
	Consider formulas $[\phi]=\{a,b\}$ and $[\mu]=\{ac,abc\}$,
	as in Example \ref{ex:dalal-postulates}, 
	over the set $\Atoms = \{a,b,c\}$ of atoms.
	Using the $\circ^{d_{\hamming},\,\max}$ operator 
	we obtain that $[\phi\circ^{d_{\hamming},\,\max}\mu]=\{abc\}$, 
	but we can show that a (putatively different) revision operator $\circ$ 
	known only to satisfy the stated postulates 
	arrives at the same conclusion. 
	It does so by first figuring out, 
	using postulates
	$\mathrm{R}_{1}$, 
	$\mathrm{R}_{3}$-$\mathrm{R}_{6}$,
	$\mathrm{R}_{\Neutrality}$,
	$\mathrm{R}_{\Addition}$,
	$\mathrm{R}_{\Flipping}$
 	that
	$[\alpha\circ f_a(\mu)]=\{bc\}$
	and 
	$[\alpha\circ f_b(\mu)]=\{abc\}$, 
	with $\alpha$, in this case, such that $[\alpha]=\{abc\}$
	(see Figure \ref{fig:hmax-postulates} for an illustration).
	At this point, we have obtained the (flipped versions of)
	the models of $\mu$ at maximal Hamming distance to $a$ and $b$, respectively.
	FOllowing this, we get that 
	$[\gamma_{\phi,\mu} = [\epsilon\circ\big((\alpha\circ f_a(\mu))\lor(\alpha\circ f_b(\mu))\big)]=\{bc\}$, 
	where $bc$ was obtained from $abc$ by flipping $a$.
	Postulate $\BoW$ then be recovers $abc$ through an extra flip of $a$.
\end{example}

\section{Characterizing the Hamming Surprise Min-Max Operator}\label{sec:surprise-hamming-max}
Finally, we return to the operator $\circ^{s_{\hamming},\,\max}$ and, 
using the wisdom gained in Section \ref{sec:dalal} and \ref{sec:distance-hamming-max},
provide it with an axiomatic foundation.
In doing so we pursue that same strategy as in the previous sections:
($i$) establish, axiomatically, what the revision result should be in the `base' case in which 
the prior belief is of a simple type, which can be decided by appeal to an argument using appealing notions of symmetry;
($ii$) reduce, axiomatically, an arbitrary instance $\phi\circ\mu$ of revision to the base case, 
in a manner that preserves the result of $\circ^{s_{\hamming},\,\max}$ on the given instance.

The base case for this section consists, 
as for the $\circ^{d_{\hamming},\,\max}$ operator, 
of revision when prior information is either $\epsilon$ or $\alpha$, 
and we want to make sure we employ a set of postulates that deliver the expected result:
since $\circ^{s_{\hamming},\,\max}$ behaves exactly like the Dalal and $\circ^{d_{\hamming},\,\max}$ operators
when prior information is complete, postulates $\mathrm{R}_{\Neutrality}$, $\mathrm{R}_{\Addition}$
and $\mathrm{R}_{\Flipping}$ can be used without modification (the assumption of completeness made in Section \ref{sec:dalal} pays off here).
We can also use the standard postulates $\mathrm{R}_{1}$ and $\mathrm{R}_{3}$-$\mathrm{R}_{4}$, 
which we already know $\circ^{s_{\hamming},\,\max}$ satisfies (see Proposition \ref{prop:surprise-hamming-postulates}).
Postulates $\mathrm{R}_{5}$-$\mathrm{R}_{6}$ are, however, problematic, 
since $\circ^{s_{\hamming},\,\max}$ does not satisfy them in their unrestricted form 
(also Proposition \ref{prop:surprise-hamming-postulates}).
However, the equivalence of $\circ^{s_{\hamming},\,\max}$ with the Dalal and $\circ^{d_{\hamming},\,\max}$ operators 
when prior information is complete means that we can use postulates $\mathrm{R}_{5}$ and $\mathrm{R}_{6}$,
restricted to the case when $\phi$ is complete. The restrictions are denoted $\mathrm{R}^{c}_{5}$ and $\mathrm{R}^{c}_{6}$, respectively.

The next step involves engineering a choice situation focused on $\alpha$ and $\epsilon$
that is equivalent, in terms of what gets chosen, to the mechanics of $\circ^{s_{\hamming},\,\max}$. 
This is done using a few intermediary notions, as follows. 
If $\phi$ and $\mu$ are formulas such that $[\phi]=\{v_1,\dots,v_n\}$,  
the \emph{adjunction interpretations $x_1$, \dots, $x_n$} are
interpretations consisting of completely new atoms such that $|x_i|=d_\hamming(v_i, \mu)$.
For $v_i\in[\phi]$, the \emph{corrected interpretation $v_i^*$} is defined as 
$v_i^* = v_i\cup(x_1\cup\dots x_{i-1}\cup x_{i+1}\cup\dots \cup x_n)$,
i.e., as the result of adding to $v_i$ all the adjunction interpretations, except $x_i$.
Then, the \emph{best-surprise formula $\sigma_{\phi,\mu}$ with respect to $\phi$ and $\mu$} 
is defined as:
$$
	\sigma_{\phi,\mu} = \epsilon\circ\bigg(\bigvee_{v_i\in[\phi]}\Big(\alpha\circ f_{v^*_i}(\mu)\Big)\bigg).
$$
In words, inside the main parenthesis we repeatedly revise $\alpha$ by a flipped version of $\mu$:
one revision for every model $v_i$ of $\phi$, flipping $\mu$ by the atoms in the corrected 
interpretation $v^*_i$.
The disjunction of all these revisions is then passed on to $\alpha$ for another round of revision.

The reasoning behind this definition is that it recasts the surprise min-max revision scenario 
for $[\phi]=\{v_1, \dots, v_n\}$ and $\mu$
into a min-max distance revision scenario 
for $[\phi^*] = \{v^*, \dots, v^*_n\}$
and $\mu$ (which we know how to axiomatize from Section \ref{sec:distance-hamming-max}), 
while keeping the relative ranking of the models of $\mu$ intact.
The following result makes this precise.

\begin{lemma}\label{lem:surprises-to-distances}
	If $\phi$ and $\mu$ are propositional formulas, 
	$v_i, v_k\in[\phi]$ and $w_j,w_\ell\in[\mu]$,
	then $s^{\mu}_{\hamming}(v_i,w_j)\le s_\hamming^\mu(v_k,w_\ell)$
	iff 
	$d_\hamming(v^*_i, w_j)\le d_\hamming(v^*_j, w_\ell)$.
\end{lemma}
\begin{proof}
	Take $[\phi]=\{v_1,\dots, v_n\}$, and $m_i = d_\hamming(v_i, \mu)$, for $v_i\in[\mu]$.
	We have that:
	\begin{align*}
		&s_\hamming^\mu(v_i,w_j) \le s_\hamming^\mu(v_k,w_\ell)~\textnormal{iff}\\ 
		&d_\hamming(v_i,w_j)-m_i \le d_\hamming(v_k,w_\ell)-m_k.
	\end{align*}
	We now add $\sum_{1\le r\le n}m_r$ on both sides, to get an equivalence with 
	$d_\hamming(v_i,w_j)+\sum_{1\le r\le n, r\neq i}m_r \le d_\hamming(v_k,w_\ell)+\sum_{1\le r\le n, r\neq k}m_r$.
	This, in turn, is equivalent to 
	$d_\hamming(v_i\cup(\bigcup_{1\le r\le n, r\neq i} x_r), w_j)\le d_\hamming(v_k\cup(\bigcup_{1\le r\le n, r\neq k} x_r), w_\ell)$, 
	which can be rewritten as
	$d_\hamming(v_i^*, w_j) \le d_\hamming(v^*_i, w_\ell)$
\end{proof}

\noindent 
Intuitively, the table of Hamming distances for $[\phi^*] = \{v^*, \dots, v^*_n\}$ and $\mu$ 
can be thought of as obtained from the surprise table for $\phi$ and $\mu$ 
(see, e.g., Table \ref{tab:hamming-surprise-1})
by adding a constant term (i.e., $\sum_{1\le r\le n}m_r$) to every entry,
a transformation that does not modify the relationships between the values:
the $v_i^*$ are the interpretations that induce the appropriate distances.
This ensures that the models of $\sigma_{\phi,\mu}$, 
obtained through a min-max distance type of postulate, 
correspond to models of $\mu$ that minimize maximum surprise with respect to $\phi$ and relative to $\mu$, 
and warrants the following postulate, called \emph{Best-of-Worst-Surpise}:

\begin{description}
	\item[($\mathrm{R}_{\BoWS}$)] $\phi\circ\mu\equiv\bigg(\bigvee_{v\in[\phi]} f_{v^\ast}(\sigma_{\phi,\mu})\bigg)\land\mu$.
\end{description}

\noindent 
As expected, the $\mathrm{R}_{\BoWS}$ postulate delivers exactly those models of $\mu$
that minimize maximum surprise, and underpins the final characterization result.

\begin{theorem}\label{th:characterization-shmax}
	A revision operator $\circ$ satisfies postulates 
	$\mathrm{R}_{1}$, 
	$\mathrm{R}_{3}$-$\mathrm{R}_{4}$,
	$\mathrm{R}^c_{5}$-$\mathrm{R}^c_{6}$, 
	$\mathrm{R}_{\Neutrality}$,
	$\mathrm{R}_{\Addition}$,
	$\mathrm{R}_{\Flipping}$
	and 
	$\mathrm{R}_{\BoWS}$ 
	iff $\circ \equiv \circ^{s_{\hamming},\,\max}$.
\end{theorem}



\noindent 
The following example illustrates the way in which postulate $\mathrm{R}_{\BoWS}$
obtains the revision result. 

\begin{example}\label{ex:shmax}
	Consider, again, formulas $[\phi]=\{a,b\}$ and $[\mu]=\{ac,abc\}$.
	We have that $[\phi\circ^{s_{\hamming},\,\max}\mu]=\{ac,abc\}$. 
	Assuming we are working with an operator $\circ$ of which the only thing we know 
	is that it satisfies the postulates in Theorem \ref{th:characterization-shmax}, 
	we notice that $d_\hamming(a, \mu)=1$ and $d_\hamming(b, \mu)=2$. 
	The postulates then direct us to compute the Hamming distance min-max result 
	for $[\phi^*] = \{ayz, bx\}$ and $\mu$, 
	with $x$ and $yz$ as the adjunction interpretations.
	The result obtained in this way is exactly $\{ac,abc\}$.
\end{example}

		
		
		
		

\section{Conclusion}\label{sec:conclusion}
We have introduced the Hamming surprise min-max operator $\circ^{s_{\hamming},\,\max}$, a revision operator 
that minimizes surprise relative to the prior belief as well as the newly acquired information. 
We have shown that, even though $\circ^{s_{\hamming},\,\max}$ does not satisfy all standard KM revision postulates, 
it is underpinned, in its choice behavior, by principles similar to those guiding established revision operators, 
among them appealing symmetry notions such as invariance under renamings and flips.
When unearthed and formulated as logical postulates, these principles (or slight variations thereof) 
turned out to be powerful enough to fully characterize not just the surprise operator, 
but also the existing Dalal and Hamming distance min-max operator.

One obvious direction for future work lies in taking the idea of context dependence further:
what other aspects of the environment influence an agent's plausibility rankings? 
Things that come to mind are issues of trust, the `strangeness' of the new information, or peer effects.
An alternative is to 
exploit the bottom-up, DIY nature of some of the postulates presented here in order to 
construct a framework, similar to that employed in collective decision-making \cite{CaillouxE16}, 
for offering \emph{justifications} for revision results, 
i.e., human-readable and at the same time rigorous step-by-step arguments 
for how to obtain a particular result, starting from a specific set of postulates.
Finally, the assumptions embedded in the present treatment call for taking the epistemic stance seriously, 
and investigating the relative worth of the various revision operators with respect to recovering 
the ground truth.

\clearpage
\bibliographystyle{kr}
\bibliography{master-bibliography}

\end{document}